\documentclass[10pt,twocolumn]{article}

\usepackage[margin=1.75cm]{geometry}
\usepackage{times}
\usepackage{nopageno}

\usepackage[table,dvipsnames]{xcolor}

\usepackage{epsfig}
\usepackage{graphicx}
\usepackage{amsmath}

\usepackage{wrapfig}
\usepackage{amsthm}
\usepackage{amssymb}
\usepackage{xspace}
\usepackage{bbold}
\usepackage[nice]{nicefrac}
\usepackage{mathtools}
\usepackage{booktabs}
\usepackage{fontawesome}
\usepackage[font={footnotesize}]{caption}
\usepackage{subcaption}
\usepackage{stfloats}
\usepackage{pbox}

\usepackage[
    pagebackref=true,
    breaklinks=true,
    letterpaper=true,
    colorlinks,
    bookmarks=false
]{hyperref}

\definecolor{_fbteal3}{HTML}{e6f5f0}
\newcommand{\norm}[1]{\left\lVert#1\right\rVert}
\newcommand{\defeq}{\coloneqq}
\newcommand{\R}{\mathbb{R}}

\newcommand{\putalg}{\textsc{paws}\xspace}

\theoremstyle{definition}

\newtheorem{proposition}{Proposition}
\newtheorem{assumption}{Assumption}

\begin{document}

\title{\Large\bf
    Semi-Supervised Learning of Visual Features by Non-Parametrically\\Predicting View Assignments with Support Samples}
\author{\large
    Mahmoud Assran$^{1,3,4}$\quad Mathilde Caron$^{1,2}$\quad Ishan Misra$^{1}$\quad Piotr Bojanowski$^{1}$\quad Armand Joulin$^{1}$\\[2mm]
    \large $^*$Nicolas Ballas$^{1}$\quad $^*$Michael Rabbat$^{1,3}$\\[4mm]
    \normalsize
    $^{1}$Facebook AI Research\quad\quad $^{2}$Inria Univ.~Grenoble Alpes\quad\quad $^{3}$Mila -- Quebec AI Institute\quad\quad $^{4}$McGill University\\
    \normalsize\tt
    \{massran, mathilde, imisra, bojanowski, ajoulin, ballasn, mikerabbat\}@fb.com
}
\date{}
\maketitle

\begin{abstract}
This paper proposes a novel method of learning by \underline{p}redicting view \underline{a}ssignments \underline{w}ith \underline{s}upport samples (\putalg).
The method trains a model to minimize a consistency loss, which ensures that different views of the same unlabeled instance are assigned similar pseudo-labels.
The pseudo-labels are generated non-parametrically, by comparing the representations of the image views to those of a set of randomly sampled labeled images.
The distance between the view representations and labeled representations is used to provide a weighting over class labels, which we interpret as a soft pseudo-label.
By non-parametrically incorporating labeled samples in this way, \putalg extends the distance-metric loss used in self-supervised methods such as BYOL and SwAV to the semi-supervised setting.
Despite the simplicity of the approach, \putalg outperforms other semi-supervised methods across architectures, setting a new state-of-the-art for a ResNet-50 on ImageNet trained with either $10\%$ or $1\%$ of the labels, reaching $75.5\%$ and $66.5\%$ top-1 respectively.
\putalg requires $4\times$ to $12\times$ less training than the previous best methods.
\let\thefootnote\relax\footnotetext{$^*$Co-last author}
\let\thefootnote\relax\footnotetext{\ \ Code: \href{https://github.com/facebookresearch/suncet}{github.com/facebookresearch/suncet}}
\end{abstract}

\section{Introduction}
\begin{figure}[t]
    \centering
    \includegraphics[width=0.95\linewidth]{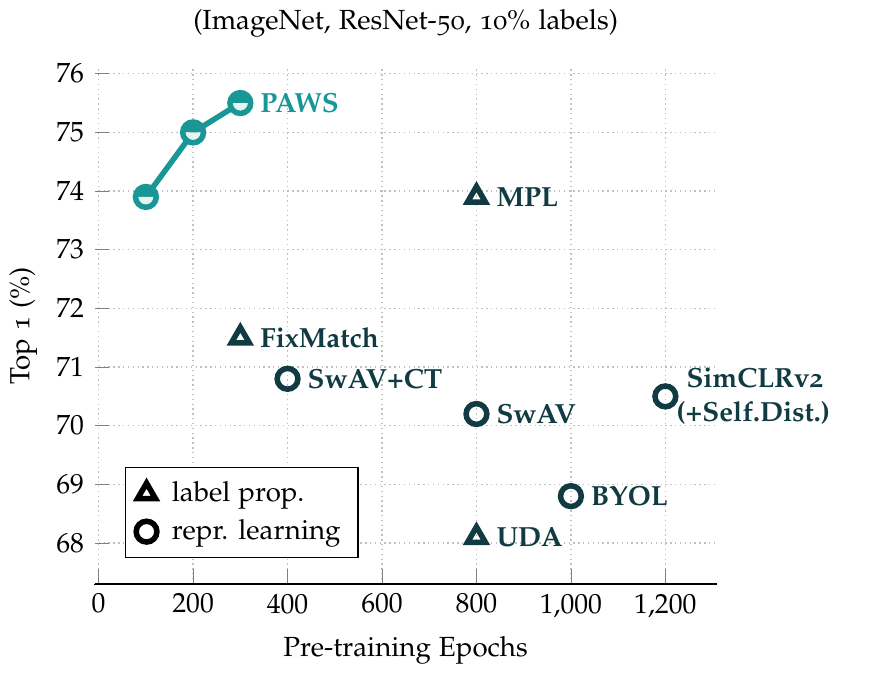}
    \caption{Training a ResNet-50 on ImageNet when only 10\% of the training set is labeled. The figure shows top-1 validation accuracy as a function of the number of training epochs. The proposed method, \putalg, achieves higher accuracy than previous work while requiring significantly fewer training epochs. Concretely, 100 epochs of \putalg training takes less than 8.5 hours using 64 \textsc{nvidia} V100-16G GPUs.}
    \label{fig:rn50_top1_10per}
\end{figure}
Learning with less labeled data has been a longstanding challenge of computer vision and machine learning research.
One popular approach for learning with few labels is to first perform unsupervised pre-training on a large dataset followed by supervised fine-tuning on the small set of available labels.
Self-supervised methods generally adhere to this paradigm (e.g., see~\cite{chen2020big} for an analysis in the context of semi-supervised learning), and they have demonstrated competitive performance on semi-supervised learning benchmarks across a wide range of self-supervised pre-training strategies~\cite{chen2020simple, chen2020big, caron2020unsupervised, grill2020bootstrap}.
However, the self-supervised paradigm also requires substantially more computational effort than other approaches and does not make use of labeled data when it is available.

An alternative line of work suggests to use available labeled data to generate pseudo-labels for the unlabeled data, and then train a model using the labeled and pseudo-labeled data~\cite{zoph2020rethinking,pham2020meta,xie2020self,tarvainen2017mean,berthelot2019mixmatch,berthelot2019remixmatch,sohn2020fixmatch}.
This begs the question, can we get the best of both worlds, leveraging labeled data throughout training while also building on advances in self-supervised learning?

This paper proposes a novel method of learning by \underline{p}redicting view \underline{a}ssignments \underline{w}ith \underline{s}upport samples (\putalg).
The method trains a model to minimize a consistency loss, which ensures that different views of the same unlabeled instance are assigned similar pseudo-labels.
The pseudo-labels are generated non-parametrically, by comparing the representations of the image views to those of a set of randomly sampled labeled images.
The distance between the view representations and labeled representations is used to provide a weighting over class labels, which we interpret as a soft pseudo-label.
By non-parametrically incorporating labeled samples in this way, \putalg extends the distance-metric loss in self-supervised methods such as BYOL~\cite{grill2020bootstrap} and SwAV~\cite{caron2020unsupervised} to the semi-supervised setting.

Despite the simplicity of the approach, \putalg outperforms other semi-supervised methods across architectures, setting a new state-of-the-art for a ResNet-50 trained on ImageNet with either $10\%$ or $1\%$ of the training instances labeled, achieving $75\%$ and $66\%$ top-1 respectively. 
Moreover, this is achieved with only 200 epochs of training, which is $4\times$ less than that of the previous best method.
The same conclusion holds when training with wider ResNet architectures as well (i.e., ResNet-50 $2\times$ or $4\times$). 

\section{Related Work}

\paragraph{Semi-supervised learning.}
One procedure to simultaneously learn with both labeled and unlabeled data is to combine a supervised loss on the labeled samples with an unsupervised loss on the unlabeled samples.
For example,~\cite{grandvalet2006entropy, miyato2018virtual, verma2019interpolation} train a model by adding an unsupervised regularization term to a supervised cross-entropy loss.
Similarly, UDA~\cite{xie2019unsupervised} adds a supervised cross-entropy loss to an appropriately weighted unsupervised regularization term.
Likewise, S4L~\cite{zhai2019s4l} adds a supervised cross-entropy loss to a weighted mixture of self-supervised pretext loss terms.
This idea of an adding a supervised cross-entropy loss to an unsupervised instance-based loss has also been exploited to learn representations suitable for both image classification and instance recognition~\cite{berman2019multigrain}.

There is also a family of semi-supervised methods related to self-training~\cite{zoph2020rethinking} that explicitly generate pseudo-labels for the unlabeled samples and that optimize prediction accuracy on both the ground truth labels (for the labeled samples) and the pseudo-labels (for the unlabeled samples).
For example Pseudo-Label~\cite{lee2013pseudo} and earlier related methods~\cite{scudder1965probability,yarowsky1995unsupervised,riloff1996automatically} first train a model on the labeled samples, use this model to assign pseudo-labels to unlabeled samples, and then re-train the model using both the labeled and unlabeled samples.
The MixMatch trilogy of work~\cite{berthelot2019remixmatch,berthelot2019mixmatch,sohn2020fixmatch} operates similarly, but generates the pseudo-labels in an online fashion.
Specifically, FixMatch~\cite{sohn2020fixmatch} trains with a supervised cross-entropy loss while simultaneously making predictions on weakly augmented unlabeled images.
When the unsupervised predictions are confident enough, they are used as pseudo-labels for strongly augmented views of those same unlabeled images.

Another closely related line of work in self-training uses an explicit teacher-student configuration.
For example, Mean Teacher~\cite{tarvainen2017mean} and Noisy Student~\cite{xie2020self} use a teacher network to assign pseudo-labels to unlabeled samples, which are then used to train a student network.
Similarly, MPL~\cite{pham2020meta} uses a teacher network to pseudo-label unlabeled images for a student network. The student then performs an update by minimizing its prediction error with respect to the teacher's pseudo-label.
Subsequently, the student is evaluated on a mini-batch of labeled samples, and the teacher network is updated using a meta-learning loss based on the student's evaluation performance.
In MPL, the overall teacher update consists of the combination of the student's meta-learning loss plus a separate UDA loss.
After self-training, the MPL student model is subsequently fine-tuned on the labeled samples using a standard cross-entropy loss.

There is also the Co-training framework~\cite{blum1998combining} which bears a coarse resemblance to the self-training procedure, but posseses notable differences.
Specifically, Co-training learns a separate feature extractor on each (conditionally independent) view of the data, combines the predictions of the different feature extractors, and alternates between pseudo-labeling a subset of the data and training on the generated pseudo-labels.

\paragraph{Few-shot learning.}
In few-shot classification, a network must be adapted to learn to recognize new classes when given only a few labeled examples of these classes~\cite{vinyals2016matching,snell2017prototypical,ravi2016optimization, lake2017building}.
One common approach, which is adopted by Matching Networks~\cite{vinyals2016matching} and Prototypical Networks~\cite{snell2017prototypical}, is to learn a metric space to embed the data. A differentiable nearest-neighbour classifier is then  used in this space to predict the class of a query point given some labeled data-points in the support set~\cite{vinyals2016matching,snell2017prototypical}.
Although there are few-shot approaches that learn entirely from unsupervised data~\cite{hsu2018unsupervised}, the majority train using labeled data, which is in contrast to the self-supervised approaches discussed next.

\paragraph{Self-supervised learning.} Major advances have been made in learning useful image representations from unlabeled data. Some methods take the approach of incorporating domain-specific knowledge in the form of specific pre-training tasks, such as solving jigsaws~\cite{misra2020self}. More recent success has been achieved by contrasting multiple views of an image~\cite{chen2020simple,he2019moco,chen2020mocov2}, where the views come from different random augmentations. Such methods aim to learn a mapping from images to a representation space such that different views of the same image have similar representations. Various approaches have been proposed to avoid the trivial solution of collapsing all images to the same point, including contrasting negative samples~\cite{chen2020simple} and using Sinkhorn-Knopp normalization~\cite{asano2019self, caron2020unsupervised}.

It has been demonstrated that self-supervised pre-training produces image representations that can be leveraged effectively for semi-supervised learning~\cite{chen2020big}. Contrastive self-supervised pre-training generally benefits from training with very large batch sizes, containing sufficiently many positive and negative examples, and consequently is very computationally expensive, e.g., requiring between 800--1000 epochs of pre-training to learn state-of-the-art representations on ImageNet.
Some recent works have demonstrated that the batch-size requirements can be reduced at the expense of maintaining an additional memory bank~\cite{chen2020exploring, he2019moco,chen2020mocov2,grill2020bootstrap,caron2020unsupervised}.
Further performance benefits have been obtained by distilling very large pre-trained teacher models to smaller student models~\cite{chen2020big}. In contrast, \putalg only trains with positive examples, and leverages available annotated data during pre-training to significantly reduce the amount of pre-training required.

\section{Methodology}
\label{sec:methodology}
We consider a large dataset of unlabeled images $\mathcal{D}=(\mathbf{x}_i)_{i\in [1, N]}$ and a small support dataset of annotated images  $\mathcal{S}=({\mathbf{x}_{s}}_i, y_i)_{i\in [1, M]}$, with $M \ll N$.\footnote{Note that the images in the support set $\mathcal{S}$ may overlap with the images in the dataset $\mathcal{D}$.} Our goal is to learn image representations by leveraging both $\mathcal{D}$ and  $\mathcal{S}$ during pre-training.
After pre-training with $\mathcal{D}$ and $\mathcal{S}$, we fine-tune the learned representations using only the labeled set $\mathcal{S}$.

\subsection{High-level Description}
A schematic of the high-level pre-training approach is shown in Figure~\ref{fig:method}.
Given an image $\mathbf{x}_i$ from $\mathcal{D}$, we use a random set of data augmentations to generate two views, an anchor view $\mathbf{\hat{x}}_i$, and an associated positive view $\mathbf{\hat{x}}^+_i$.
Learning proceeds by non-parametrically assigning soft pseudo-labels to the anchor and positive view and subsequently minimizing the cross-entropy $H(\cdot, \cdot)$ between them.
\begin{figure}[t]
    \centering
    {\centering
        \includegraphics[width=\linewidth]{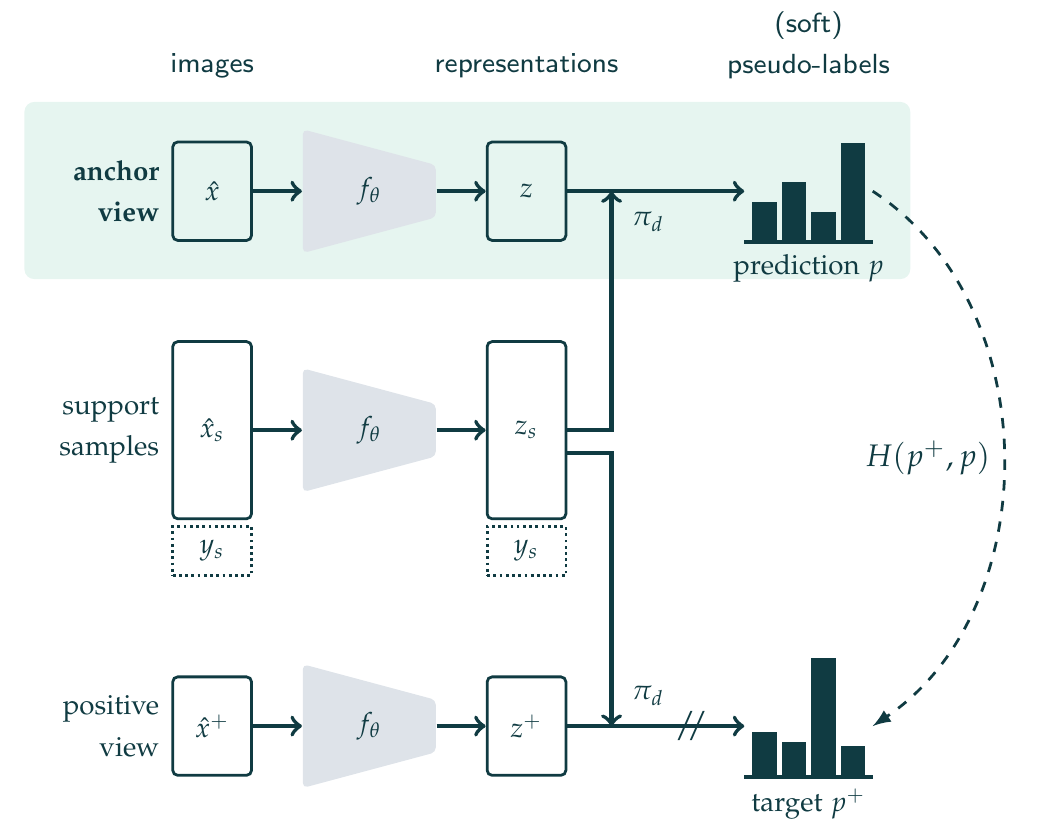}
    }\caption{{\putalg.} The method assigns soft pseudo-labels to an anchor view of an image and an associated positive view, and subsequently minimizes the cross-entropy $H$ between them. The soft pseudo-labels are generated using a differentiable similarity classifier $\pi_d$ that measures the similarity to a mini-batch of labeled support samples, and outputs a soft class distribution.
    Positive views are created using data-augmentations of the anchor view. Since the trivial collapse of all representations to a single vector would lead to high-entropy predictions by the similarity classifier, sharpening the target pseudo-labels is sufficient to eliminate all trivial solutions.}
    \label{fig:method}
\end{figure}

The soft pseudo-labels are generated using a differentiable similarity-based classifier $\pi_d$ that measures the similarity of a given representation to those of a mini-batch of labeled samples from the support set $\mathcal{S}$, and outputs a (soft) class label. We use a simple Soft Nearest Neighbours strategy~\cite{salakhutdinov2007learning} for the similarity classifier $\pi_d$.

\paragraph{Connection to few-shot learning.} The mini-batch of labeled samples is obtained by first sampling a subset of classes and then sampling a few instances of each class. This, along with the use of a soft nearest-neighbours strategy is similar to approaches previously used for few-shot classification~\cite{vinyals2016matching}. However, unlike~\cite{vinyals2016matching}, we do not use LSTMs or other mechanisms for encoding or accessing elements of the support set, and furthermore, we never seek to directly predict the labels of elements of the support set. Rather, the support set is only used to assign pseudo-labels to unlabeled image views, and the loss is only evaluated with respect to the pseudo-labels assigned to the unlabeled image views.

\subsection{Detailed Methodology}
Let ${\bf x} \in \R^{n \times (3\times H\times W)}$ denote a mini-batch of $n$ anchor image views, and let ${\bf x}^+ \in \R^{n \times (3\times H\times W)}$ denote the associated $n$ positive image views. Similarly, let ${\bf x}_\mathcal{S} \in \R^{m \times (3\times H\times W)}$ denote a mini-batch of $m$ support images drawn from $\mathcal{S}$ with one-hot class labels ${\bf y}_\mathcal{S} \in \R^{m \times K}$, where $K$ is the number of classes.

\paragraph{Encoder.}
Given a parameterized encoder, denoted by $f_\theta:\R^{3 \times H \times W} \rightarrow \R^d$, let ${\bf z} \in \R^{n \times d}$ and ${\bf z}^+ \in \R^{n \times d}$ denote the representations computed from ${\bf x}$ and ${\bf x}^+$, respectively, and let ${\bf z}_\mathcal{S} \in \R^{m \times d}$ denote the $m$ support representations computed from ${\bf x}_\mathcal{S}$. In our experiments below, the encoder will be the trunk of a deep residual network~\cite{he2016deep}.
The $i^{th}$ representation in the mini-batch $\bf z$ is written as a \emph{row-vector} $z_i \in \R^{1\times d}$, and its associated positive view in the mini-batch is denoted $z^+_i$; i.e., $z_i = f_\theta({\bf x}_i)$ and $z_i^+ = f_\theta({\bf x}_i^+)$.
For a scalar-valued similarity function $d(\cdot, \cdot) \geq 0$, the similarity classifier $\pi_d(\cdot, \cdot)$ is given by
\begin{equation*}
    \pi_d(z_i, {\bf z}_\mathcal{S})
    = \sum _{({z_s}_j, y_j) \in {\bf z_\mathcal{S}}}
    \left(\frac{d(z_i, {z_s}_j)}{\sum_{{z_s}_k \in {\bf z_\mathcal{S}}} d(z_i, {z_s}_k)}\right) y_j
\end{equation*}
where $y_j$ is the one-hot ground truth label vector associated with the $j^{th}$ row vector ${z_s}_j$ from ${\bf z_\mathcal{S}}$.

\paragraph{Similarity metric and predictions.} In this work, we take the similarity metric $d(a, b)$ to be $\exp(\nicefrac{a^T b}{\norm{a}\norm{b}\tau})$, the exponential temperature-scaled cosine. For L2-normalized representations, the similarity classifier $\pi_d(\cdot, \cdot)$ can be concisely written as
\[
    p_i \defeq \pi_d(z_i, {\bf z_\mathcal{S}}) = \sigma_\tau(z_i {\bf z_{\mathcal{S}}^\top}){\bf y_{\mathcal{S}}},
\]
where $\sigma_\tau(\cdot)$ is the softmax with temperature $\tau > 0$, and $p_i \in [0,1]^K$ is the prediction for representation $z_i$.\footnote{Specifically, given a vector $a \in \R^K$, the softmax $\sigma_\tau(a) \in [0,1]^K$ is defined as $[\sigma_\tau(a)]_k \defeq \frac{\exp\left(a_k/\tau\right)}{\sum^K_{j=1}\exp\left(a_j/\tau\right)}$ for $k=1,\ldots,K$.} The positive view predictions $p_i^+$ are calculated similarly from representations $z_i^+$.

To avoid representation collapse, rather than contrast negative samples or incorporate Sinkhorn-Knopp normalization, we compare the prediction of one view with the sharpened prediction of the other view.
We define the sharpening function $\rho(\cdot)$ with temperature $T > 0$ as 
\[
     [\rho(p_i)]_k \defeq \frac{{[p_i]_k}^{\nicefrac{1}{T}}}{\sum^K_{j=1} {[p_i]_j}^{\nicefrac{1}{T}}},
      \qquad k=1,\ldots,K.
\]
Sharpening the targets encourages the network to produce confident predictions.
As will be clear in Section~\ref{sec:theory}, sharpening the targets is provably sufficient to eliminate collapsing solutions in the \putalg framework. Empirically, we have observed that training without sharpening can result in collapsing solutions.

Note that in the case where the support set contains only one instance per sampled class, sharpening the target predictions is equivalent to using a lower temperature in the cosine similarity between the unlabeled representation and support representations.
However, when the sampled support set contains more than one instance per sampled class, then sharpening the target predictions is actually different from adjusting the cosine temperature.
In this case, it is preferable to sharpen the target predictions rather than use a different temperature in the cosine similarity, since changing the cosine temperature can significantly affect the accuracy of the similarity classifier $\pi_d$.

\paragraph{Training objective.} To train the encoder, we penalize when the predictions $p_i$ and $p_i^+$ of two views of the same image are different. As mentioned above, we compare the prediction of one view with the sharpened prediction of the other view; i.e., $H(\rho(p_i), p_i^+) + H(\rho(p_i^+), p_i)$.

We also incorporate a regularization term to encourage the image view representations to utilize the full set of classes represented in the support set.
Let $\overline{p} \defeq \frac{1}{2n}\sum^{n}_{i=1}\big(\rho(p_i) + \rho(p_i^+)\big)$ denote the average of the sharpened predictions across all unlabeled representations. The regularization term, which we refer to as \textit{mean entropy maximization} (\textsc{me-max}), seeks to maximize the entropy of $\overline{p}$, denoted $H(\overline{p})$. That is, while the individual predictions are encouraged to be confident, the average prediction is encouraged to be close to the uniform distribution.
The \textsc{me-max} regularizer has previously been used in the discriminative unsupervised clustering community for balancing learned cluster sizes (see, e.g.,~\cite{joulin2012convex}).

Thus, the overall objective to be minimized when training the parameters $\theta$ of the encoder $f_\theta$ is
\begin{equation}
    \label{eq:objective}
    \frac{1}{2n} \sum^{n}_{i=1} \left(H(\rho(p^+_i), p_i) + H(\rho(p_i), p_i^+)\right) - H(\overline{p}).
\end{equation}
Note that we only differentiate the cross-entropy loss terms with respect to the predictions $p_i$ and $p^+_i$, and not the sharpened targets $\rho(p_i)$ and $\rho(p^+_i)$.

The discussion so far has assumed that we only generate two views for each unlabeled image. One could generate more than two views, in which case we sum the loss over all views and take the target to be the average prediction across the other views of the same image.

The proposed approach seeks to improve on existing self-supervised approaches for semi-supervised learning by: (i) efficiently using available task information, and (ii) addressing representation collapse.
On the first issue, since the similarity classifier is differentiable, we evaluate gradients with respect to the labeled samples, but do not directly optimize prediction accuracy on the ground truth labels to avoid overfitting. 
On the second issue, since the trivial collapse of all representations to a single vector would lead to high-entropy predictions by the similarity classifier, sharpening the target pseudo-labels is sufficient to eliminate all trivial solutions as we will demonstrate in Section~\ref{sec:theory}.

\paragraph{Neural architectures with external memory.}
\putalg can be interpreted as a neural network architecture with an external memory. Typically, in those architectures, a differentiable neural attention mechanism is used to read and access a memory space which contains a set of elements that are relevant to the task at hand.
In \putalg, the support representations ${\bf z_\mathcal{S}}$ of labeled images characterize the external memory of the network, while the non-parameteric classifier $\pi_d$ corresponds to the soft-attention operation that retrieves memory elements given a query $z_i$.
From this perspective, \putalg optimizes an encoder network such that two views of the same image activate the same elements in the memory.
Moreover, by randomly sampling a subset of labeled images to use as the support set at each iteration, \putalg avoids developing a strong dependence on any particular elements in the memory.

\paragraph{Assimilation \& Accommodation.}
\putalg also has connections to Piaget's Constructivist learning theory of \emph{assimilation \& accommodation}~\cite{piaget1964cognitive}, which provided grounding for work in cybernetics~\cite[Chapter VII]{boden1980jean}.\footnote{This connection did not readily carry-over to Artificial Intelligence (AI) in the 70's due to the largely symbolic nature of AI approaches at the time; e.g., it was not obvious how to represent the near infinite variations of a hand-drawn curve in a single concise representation; an issue which is now largely resolved by gradient-based learning and modern neural network architectures.}
At the heart of Constructivism is the idea that every individual possesses representations relating to distinct semantic concepts that are updated through the process of \emph{assimilation and accommodation}.
During assimilation, the mind adapts its representations of new observations to fit its \emph{past} observations, while during accommodation, the representations of \emph{past} observations are updated to account for the new observations (cf.~Appendix~\ref{apndx:historical}).
In the \putalg procedure, backpropagating with respect to the image views can be seen as a process of assimilation, ensuring that new observations (the image views) are consistent with the current schemata (the support representations).
Similarly, backpropagating  with respect to the support samples can be seen as a process of accommodation, ensuring that the current schemata (the support representations) are effective at describing the new observations (the image views).

\section{Theoretical Guarantees}
\label{sec:theory}
Next we show that \putalg is guaranteed to avoid the trivial collapse of representations under the following assumptions.

\begin{assumption}[Class Balanced Sampling]
\label{ass:balanced}
Each mini-batch of labeled support samples contains an equal number of instances from each of the sampled classes.
\end{assumption}
\begin{assumption}[Target Sharpening]
\label{ass:sharp}
The target $p^+$ is sharpened, such that it is not equal to the uniform distribution.
\end{assumption}
\begin{proposition}[Non-Collapsing Representations]
\label{prop:collapse}
Suppose Assumptions~\ref{ass:balanced} and~\ref{ass:sharp} hold.
If $f_\theta$ is such that the representations collapse, i.e., $z_i = z$ for all $z_i \in \mathcal{S}$, then $\norm{\nabla_\theta H(p^+,p)} > 0$.
\end{proposition}
\begin{proof}
Since $z = z_i$ for all $z_i \in \mathcal{S}$, it holds that $d(z, z_i) = d(z, z_j)$ for all $z_i, z_j \in \mathcal{S}$.
Therefore $p \defeq \pi_d\left(z, \mathcal{S}\right) = \nicefrac{1}{n} \sum_{(z_i, y_i)} y_i$, where $y_i$ is the one-hot class label for the representation $z_i$.
Let $K$ denote the number of classes represented in the mini-batch of support samples.
By Assumption~\ref{ass:balanced}, since the mini-batch of support samples contains an equal number of instances from each sampled class, it follows that there are $\nicefrac{n}{K}$ instances for each of the $K$ represented classes.
Therefore, the prediction $p$ further simplifies to $\frac{1}{n} \left( {\bf{1}_K} \frac{n}{K} \right) = \frac{1}{K} {\bf{1}_K}$, the uniform distribution over the $K$ classes.
However, by Assumption~\ref{ass:sharp}, the targets $p^+$ are sharpened such that they are not equal to the uniform distribution.
Therefore, $p \neq p^+$, from which it follows that $\norm{\nabla H(p^+,p)} > 0$.
\end{proof}
Proposition~\ref{prop:collapse} provides a theoretical guarantee that the proposed method is immune to the trivial collapse of representations.
It is also straightforward to extend Proposition~\ref{prop:collapse} to accommodate certain popular transformations of the labels $y_i$, such as label smoothing.
In short, the underlying principle is that collapsing representations result in high entropy predictions under the non-parametric similarity classifier, but the targets are always low-entropy (because we sharpen them), and so collapsing all representations to a single vector is not a stationary point of the training dynamics.

Note that the sharpening function defined in Section~\ref{sec:methodology} may not always satisfy Assumption~\ref{ass:sharp}, unless one introduces a simple
tie-breaking mechanism.
However, in practice, such a mechanism is not necessary as the targets never become uniform (since we apply sharpening from the start of the training).
There are also alternative strategies to guarantee the non-collapse of representations without making the target-sharpening assumption, such as by directly using the available class labels for prediction or adding an entropy-minimization term; see Appendix~\ref{apndx:theory} for more details.

\section{Implementation Details}
\label{sec:implementation}
We first pre-train a network using \putalg, and then fine-tune the learned representations for the classification task using only the labeled samples. We also report results using the pre-trained representations directly in a nearest-neighbour classifier.

We adopt similar hyper-parameter settings that have previously been reported in the self-supervised literature~\cite{chen2020big, chen2020simple, chen2020exploring, caron2020unsupervised, grill2020bootstrap}.
Specifically, for pre-training, we use the LARS optimizer~\cite{you2017large} with a momentum value of 0.9, weight decay $10^{-6}$, cosine-similarity temperature of $\tau = 0.1$, and batch-size of 4096. We linearly warm-up the learning-rate from 0.3 to 6.4 during the first 10 epochs of pre-training, and decay it following a cosine schedule~\cite{loshchilov2016sgdr} thereafter.

To construct the different image views, we use the multi-crop strategy from SwAV~\cite{caron2020unsupervised}, generating two large crops ($224\times224$), and six small crops ($96\times96$) of each unlabeled image.
Each small crop has two positive views (the two large crops), while each large crop has only one positive view (the other large crop).\footnote{The target for the small crops is the average of the large crop predictions.}
To construct the support mini-batch at each iteration, we also randomly sample 6720 images, comprising 960 classes and 7 images per class, from the labeled set.
For all sampled images (both unlabeled images and support images), we apply the SimCLR data-augmentations~\cite{chen2020simple, chen2020big}, specifically random crop, horizontal flip, color distortion, and Gaussian blur.
For the sampled support images, we also apply label smoothing with a smoothing factor of $0.1$.
Lastly, for the target sharpening, we use a temperature of $T=0.25$.

Following previous self-supervised methods, the encoder $f_\theta$ in our experiments is a ResNet trunk with a 3-layer MLP projection head~\cite{chen2020big, grill2020bootstrap}.
To facilitate comparison with BYOL~\cite{grill2020bootstrap}, we also include a 2-layer MLP prediction head, $g_\zeta$, after $f_\theta$, before computing the anchor predictions.
Specifically, the representations $z$ and ${\bf z_\mathcal{S}}$ are fed into $g_\zeta$ before computing their cosine similarity.
While this prediction head is included in our default setup for consistency with previous work, the ablation experiments below (see Table~\ref{tb:prediction-head}), show that \putalg also works well without it.
Similar to previous self-supervised methods~\cite{chen2020big, chen2020simple, grill2020bootstrap}, we also use global batch normalization during pre-training, and exclude the bias and batch-norm parameters from weight decay and LARS adaptation.

After pre-training, we fine-tune a linear classifier from the first layer of the projection head in the encoder $f_\theta$, and follow the evaluation protocol of BYOL~\cite{grill2020bootstrap}.
Specifically, we simultaneously fine-tune the encoder/classifier weights using the available labeled samples and a standard supervised cross-entropy loss.
See Appendix~\ref{apndx:implementation} for more details, and Section~\ref{sec:ablation} for ablation experiments.

We also report the results of using the pre-trained representations directly in a nearest-neighbour classifier (i.e., without fine-tuning).
Specifically, the nearest-neighbour classifier compares the representations of new query images to those of the available labeled data.
We refer to this approach as \putalg-\textsc{nn}.

\section{Main Results}
\label{sec:results}
\begin{table}[t]
    \footnotesize
    \centering
    {\small {ResNet-50}\\[2mm]
    \begin{tabular}{l r c c}
        & & \multicolumn{2}{c}{\bf\small Top 1}\\
        \bf\small Method & \bf\small Epochs & \bf\small 1\% & \bf\small 10\% \\\toprule
        \multicolumn{4}{l}{\footnotesize\itshape Methods using label propagation:}\\[1mm]
        UDA~\cite{xie2019unsupervised} & 800 & -- & 68.1 \\
        FixMatch~\cite{sohn2020fixmatch} & 300 & -- & 71.5 \\
        MPL~\cite{pham2020meta} & $^\star$800 & -- & 73.9 \\\midrule
        \multicolumn{4}{l}{\footnotesize\itshape Methods using only representation learning:}\\[1mm]
        BYOL~\cite{grill2020bootstrap} & 1000 & 53.2 & 68.8 \\
        SwAV~\cite{caron2020unsupervised} & 800 & 53.9 & 70.2 \\
        SwAV+CT~\cite{assran2020recovering} & 400 & -- & 70.8 \\
        SimCLRv2~\cite{chen2020big} & 800 & 57.9 & 68.4 \\
        SimCLRv2 (+Self.Dist.)~\cite{chen2020big} & 1200 & 60.0 & 70.5 \\
        \rowcolor{_fbteal3}
        \putalg & \bf 100 & \bf 63.8 & \bf 73.9 \\
        \rowcolor{_fbteal3}
        \putalg & \bf 200 & \bf 66.1 & \bf 75.0 \\
        \rowcolor{_fbteal3}
        \putalg & \bf 300 & \bf 66.5 & \bf 75.5 \\ \midrule
        \multicolumn{4}{l}{\footnotesize\itshape Non-parametric classification (no fine-tuning):}\\[1mm]
        \rowcolor{_fbteal3}
        \putalg-\textsc{nn} & 100 & 61.5 & 71.0 \\
        \rowcolor{_fbteal3}
        \putalg-\textsc{nn} & 200 & 63.2 & 71.9 \\
        \rowcolor{_fbteal3}
        \putalg-\textsc{nn} & 300 & 64.2 & 73.1 \\ \bottomrule
    \end{tabular}}
    \caption{{\bfseries (ResNet-50, ImageNet)} *For label propagation methods, the number of epochs is counted with respect to the unsupervised mini-batches. *For Meta Pseudo-Labels (MPL), the number of epochs only includes the student-network updates, and does not count the additional 500,000 teacher-network updates (computationally equivalent to an additional 800 epochs) that must happen sequentially (not in parallel) with the student updates. \putalg-\textsc{nn} refers to performing nearest-neighbour classification directly using the \putalg-pretrained representations, with the labeled training samples as support, while \putalg refers to fine-tuning a classifier using the available labeled data after \putalg-pretraining.}
    \label{tb:resnet50_results}
\end{table}
\begin{table}[t]
    \footnotesize
    \centering
    {\small {Additional ResNet Architectures}\\[2mm]
    \begin{tabular}{l l r c c}
        & & & \multicolumn{2}{c}{\bf\small Top 1}\\
        \bf\small Method & \bf\small Architecture & \bf\small Epochs & \bf\small 1\% & \bf\small 10\% \\\toprule
        BYOL~\cite{grill2020bootstrap} & ResNet-50 (2$\times$) &  1000 & 62.2 & 73.5 \\
        SimCLRv2~\cite{chen2020big} & ResNet-50 (2$\times$)  & 800 & 66.3 & 73.9 \\
        \rowcolor{_fbteal3}
        \putalg & ResNet-50 (2$\times$)  & \bf 100 &  \bf 68.2 & \bf 77.0 \\
        \rowcolor{_fbteal3}
        \putalg & ResNet-50 (2$\times$) & \bf 200 & \bf 69.6 & \bf 77.8 \\\midrule
        SimCLR~\cite{chen2020simple} & ResNet-50 (4$\times$) &  1000 & 63.0 & 74.4 \\
        BYOL~\cite{grill2020bootstrap} & ResNet-50 (4$\times$) &  1000 & 69.1 & 75.7 \\
        \rowcolor{_fbteal3}
        \putalg & ResNet-50 (4$\times$) &  \bf 100 & \bf 69.8 & \bf 78.5 \\
        \rowcolor{_fbteal3}
        \putalg & ResNet-50 (4$\times$) &  \bf 200 & \bf 69.9 & \bf 79.0 \\\bottomrule
    \end{tabular}}
    \caption{Semi-supervised classification results on ImageNet when training with larger ResNet architectures.}
    \label{tb:largeresnet_results}
\end{table}
In this section we analyze the features learned by \putalg on ImageNet~\cite{russakovsky2015imagenet}.
The standard procedure for evaluating semi-supervised methods on ImageNet is to assume that some percentage of the data is labeled, and treat the rest of the data as unlabeled.
For reproducibility, we use the same $1\%$ and $10\%$ data splits used in previous works~\cite{chen2020simple, chen2020big}.

While we assume that the overall support set contains all relevant labels for the downstream task, we believe this is reasonable since the overall (labeled) support set is small and can be more easily curated.
Exploring performance in settings with class imbalance or partial coverage are beyond the scope of this paper and are left as future work.

\paragraph{Baselines.} We focus on comparing \putalg to other methods in the literature that train using the same architectures to make a fair comparison. We do not include comparisons with results that first train a larger teacher model and then distill it to a smaller student~\cite{chen2020big}. For reference, the best reported result in the literature for a ResNet-50 and 1\% or 10\% labeled data are 73.9\% and 77.5\% top-1, achieved by distilling from a ResNet-152 with 3$\times$ wider channels and selective kernels~\cite{chen2020big}. We impose this constraint on the baselines to provide a fair comparison and better isolate what factors contribute to performance improvements. It is know that using distillation in conjunction with larger architectures can result in improvements for any method, and we leave further investigation of distilling larger models pre-trained with \putalg for future work.
\begin{figure}
    \centering
    \includegraphics[width=0.95\linewidth]{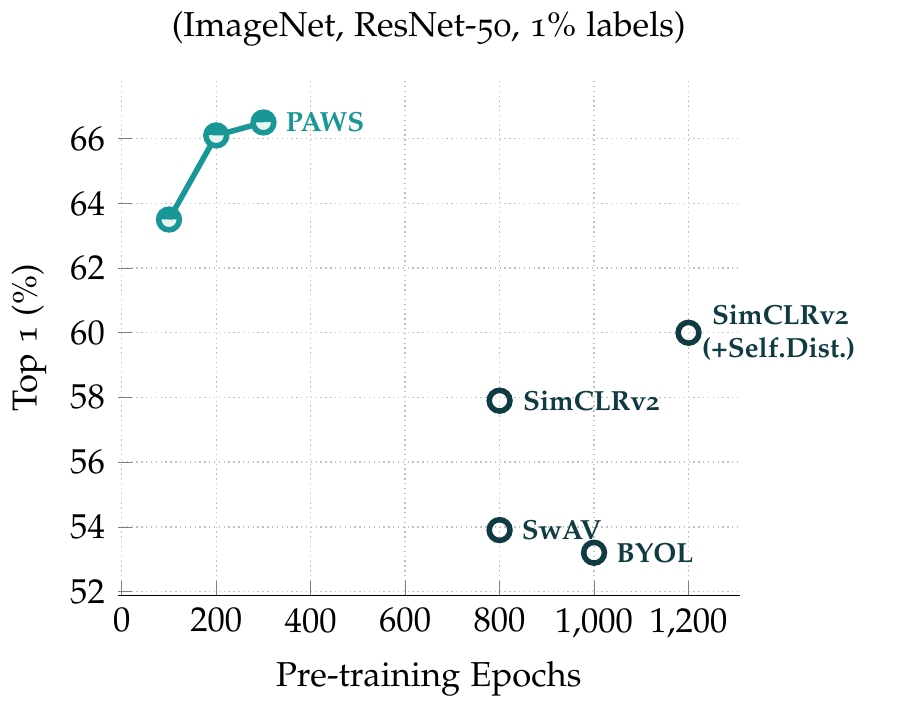}
    \caption{Training a ResNet-50 on ImageNet when only 1\% of the training set is labeled. The figure shows top-1 validation accuracy as a function of the number of training epochs. The proposed method, \putalg, achieves higher accuracy than previous work while requiring significantly fewer training epochs. Concretely, 100 epochs of \putalg training takes less than 8.5 hours using 64 \textsc{nvidia} V100-16G GPUs.}
    \label{fig:rn50_top1_1per}
\end{figure}

\paragraph{Comparison to self-supervised pre-training.}
We compare \putalg to other self-supervised pre-training approaches, namely SimCLRv2~\cite{chen2020big}, BYOL~\cite{grill2020bootstrap}, SwAV~\cite{caron2020unsupervised}, and SwAV+CT~\cite{assran2020recovering}, which simply adds a supervised contrastive-task loss to SwAV pre-training.
Results are reported in Table~\ref{tb:resnet50_results} for a ResNet-50 encoder network and in Figures~\ref{fig:rn50_top1_10per},\ref{fig:rn50_top1_1per}.
\putalg outperforms all other self-supervised representation learning approaches while using roughly $10\times$ fewer pre-training epochs.
Specifically, with just 100 epochs of pre-training, \putalg surpasses the state-of-the-art in self-supervised representation learning.
With 200 epochs of pre-training, \putalg further improves upon this result and achieves $75\%$ top-1 accuracy in the $10\%$ label setting and $66\%$ top-1 in the $1\%$ label setting, setting a new state-of-the-art for a ResNet-50. Using the pre-trained representations directly in a nearest-neighbour classifier (\putalg-\textsc{nn}) also performs surprisingly well---surpassing all other self-supervised representation learning methods---although fine-tuning increases top-1 accuracy by 1--3\%. Because \putalg with fine-tuning consistently achieves superior results compared to \putalg-\textsc{nn}, we only report results for \putalg for the remainder of the paper.

By reducing the number of pre-training epochs, \putalg can obtain significant computational savings compared to other approaches.
We illustrate this observation by comparing \putalg training time on 64 \textsc{nvidia} V100-16G GPUs to the self-supervised SwAV method trained on identical hardware~\cite{caron2020unsupervised}.
Pre-training with SwAV for 800 epochs requires $49.6$ hours, while pre-trianing with \putalg for 100 epochs only requires $8.2$ hours, and results in a $+9.9\%$ improvement in top-1 accuracy in the $1\%$ label setting, and a $+3.7\%$ improvement in top-1 accuracy in the $10\%$ label setting.
In contrast to SimCLRv2 and BYOL, the \putalg method does not use an additional momentum encoder or a memory buffer, and thereby avoids this added computational and memory overhead, but may also benefit (in terms of final model accuracy) by incorporating such innovations.

\paragraph{Comparison to semi-supervised methods.}
We also compare \putalg to other semi-supervised learning methods, namely UDA~\cite{xie2019unsupervised}, FixMatch~\cite{sohn2020fixmatch} and MPL~\cite{pham2020meta}.
Results are reported in Table~\ref{tb:resnet50_results} for a ResNet-50 encoder network in the 10\% label setting.
MPL holds the current state-of-art in semi-supervised learning, and simultaneously trains a student and teacher network for 800 epochs by adding a meta-learning loss and a teacher network to the UDA objective.
\putalg outperforms MPL, the state-of-art semi-supervised learning approach, while requiring significantly fewer training epochs.

\paragraph{Impact of larger architectures.}
We examine the impact of training larger encoder networks with \putalg pre-training.
Specifically, we pre-train ResNet-50 encoders with width multipliers of $2\times$ and $4\times$ in Table~\ref{tb:largeresnet_results}.
As expected, increasing the model capacity improves semi-supervised performance.
Specifically, pre-training a Resnet-50 (4$\times$) for 200 epochs with \putalg achieves $69.9\%$ top-1 accuracy in the $1\%$ label setting and $79.0\%$ top-1 accuracy in the $10\%$ label setting.
We expect increasing the model capacity further to yield additional performance improvements.
In general, results with the larger models are consistent with previous observations; \putalg pre-training outperforms other methods using similar architectures, while requiring significantly fewer pre-training epochs.

\section{Ablation Study}
\label{sec:ablation}

\paragraph{Longer training.}
The results reported in Section~\ref{sec:results} illustrate the performance of \putalg after 100 and 200 pre-training epochs.
We have not observed substantial benefits to training for longer than this.
Results after pre-training longer are shown in Table~\ref{tb:longer-training} for ResNet-50 $1\times$ and $2\times$ architectures.
\begin{table}[h]
    \centering
    {\small
    \begin{tabular}{lrcc}
        & & \multicolumn{2}{c}{\bf\small Top-1}\\
        {\bf\small Architecture} & {\bf\small Epochs} & {\bf\small 1\%} & {\bf\small 10\%} \\\toprule
        ResNet-50 & 100 & 63.8 & 73.9 \\
        ResNet-50 & 200 & 66.1 & 75.0 \\
        ResNet-50 & 300 & 66.5 & 75.5 \\\midrule
        ResNet-50 (2$\times$) & 100 & 68.2 & 77.0 \\
        ResNet-50 (2$\times$) & 200 & 69.6 & 77.8 \\
        ResNet-50 (2$\times$) & 300 & 69.6 & 77.7 \\\bottomrule
    \end{tabular}}
    \caption{{\bfseries Longer Training.} Examining the impact of longer training for various ResNet architectures on ImageNet. In both 1\% and 10\% label settings, and across both ResNet-50 and ResNet-50 (2$\times$) architectures, training for more than 200 epochs is generally not necessary and only yields marginal improvements.}
    \label{tb:longer-training}
\end{table}

While \putalg does not seem to benefit from longer training, it is interesting to observe that, by contrast, \putalg-\textsc{nn}, which performs nearest neighbours classification (no fine-tuning), may benefit from longer training, as suggested by Table~\ref{tb:resnet50_results}.

\paragraph{Learning during pre-training.}
To further examine the behaviour of \putalg, we examine some metrics related to model quality during pre-training in Figure~\ref{fig:training-curves}. Figure~\ref{subf-fig:train-loss} shows the training cross-entropy loss when pre-training for 100 epochs. As expected, this loss decreases during training, indicating that the model is learning to assign similar pseudo-labels to different views of the same image.

Figure~\ref{subf-fig:additional-loss} shows two additional losses computed using the sampled mini-batch and support set during training. Here, the instance discrimination loss is the normalized temperature-scaled cross-entropy loss~\cite{chen2020simple} computed using only unlabeled samples in the minibatch, and the classification loss is supervised noise-contrastive estimation loss~\cite{assran2020recovering,khosla2020supervised} computed using only labeled samples in the support set. Note that these losses are only computed and reported to better understand \putalg pre-training, and they are not used to train the model. The decreasing instance discrimination loss (top) indicates that the model is learning representations that are invariant to the data augmentations used to construct different views. The decreasing classification loss (bottom) also indicates that the model is learning to correctly classify labeled examples in the support set, despite not directly using labeled examples as targets.
\begin{figure*}[t]
    \centering
    \begin{subfigure}{0.26\linewidth}
        \centering
        \includegraphics[width=\linewidth]{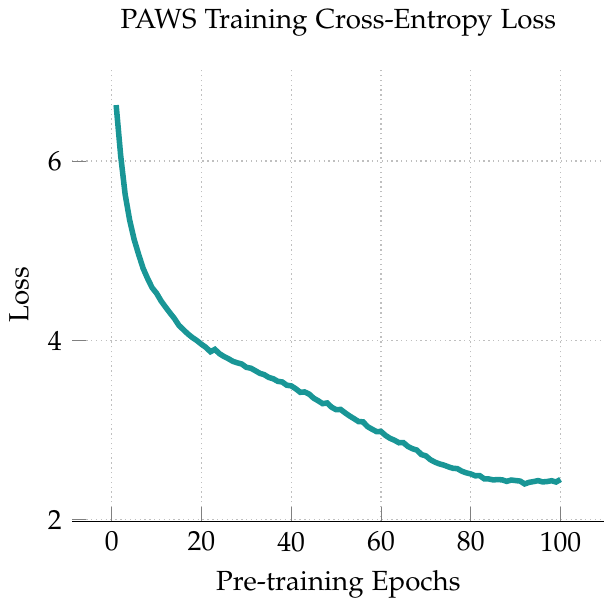}
        \caption{}
        \label{subf-fig:train-loss}
    \end{subfigure}\quad
    \begin{subfigure}{0.26\linewidth}
        \centering
        \includegraphics[width=\linewidth]{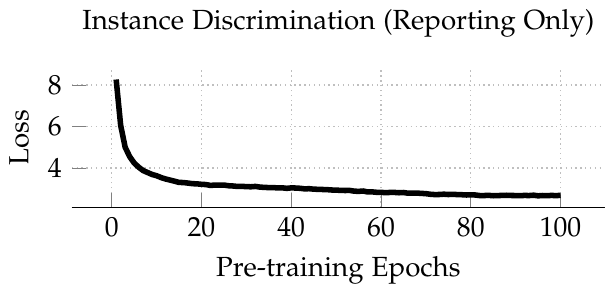}\\[2mm]
        \includegraphics[width=\linewidth]{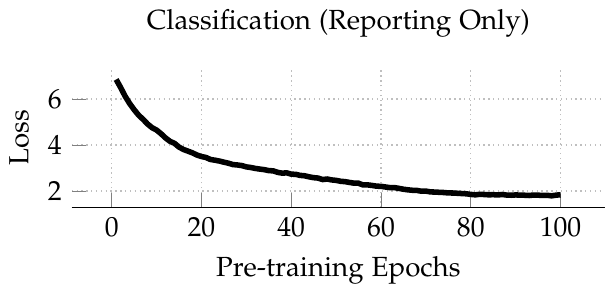}
        \caption{}
        \label{subf-fig:additional-loss}
    \end{subfigure}
    \quad{\color{lightgray}\vrule}\quad
    \begin{subfigure}{0.26\linewidth}
        \centering
        \includegraphics[width=\linewidth]{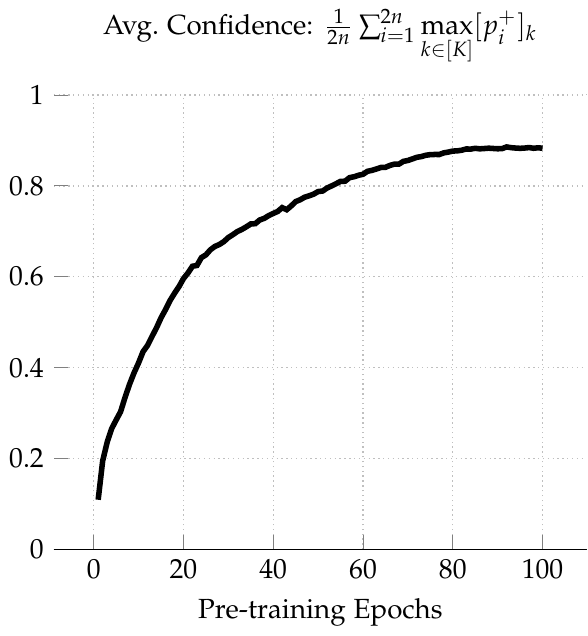}
        \caption{}
        \label{subf-fig:prob}
    \end{subfigure}
    \caption{Reporting various metric during training of a ResNet-50 on ImageNet, when 10\% of the data is labeled. Fig.\ref{subf-fig:train-loss} Cross-entropy loss between anchor view and (target) positive view during training. As expected, this loss decreases during training, indicating that the model is learning to assign similar pseudo-labels to different views of the same image.  Fig.\ref{subf-fig:additional-loss} Additional losses computed with the sampled mini-batch and support-set during training for \emph{reporting purposes only}. Specifically, no gradient is computed with respect to these losses. The decrease in the instance discrimination loss during training suggests that the model is learning representations that are invariant to the data-augmentations used for training. The decrease in the classification loss indicates that the model is learning to correctly classify the labeled support samples. Fig.\ref{subf-fig:prob} The average confidence of the $\text{argmax}$ target prediction during training. As training progresses, the model's target predictions become increasingly confident.}
    \label{fig:training-curves}
\end{figure*}

\paragraph{Support set construction.}
\putalg pre-training requires specifying how to sample a support set. At each iteration, a support set is sampled by first sampling a subset of the $K$ classes, and then sampling a certain number of images per class. We ablate the effect of these two parameters in Table~\ref{tb:support-set}. Since we experiment with ImageNet, we can sample up to 1000 classes. Overall, we observe that using a larger support set consistently improves performance. Sampling more classes and fewer samples per class is better than the contrary (cf.~bottom two rows). Note that no result is reported for 1000 classes and 16 images per class for the case of 1\% labeled data, since in that case there are only 12811 labeled images in total.
\begin{table}[h]
    \centering
    {\small
    \begin{tabular}{cccc}
        & & \multicolumn{2}{c}{\bf\small Top 1}\\
        {\bf\small Classes} & {\bf\small Imgs.~per Class} & \bf\small 1\% & \bf\small 10\% \\\toprule
        1000 & 16 & -- & 74.5 \\ 
        1000 & 12 & 63.9 & 74.2 \\
        \rowcolor{_fbteal3}
        960 & 7 & 63.8 & 73.9 \\
        960 & 4 & 63.7 & 72.0 \\
        448 & 8 & 61.8 & 70.1 \\\bottomrule
    \end{tabular}}
    \caption{{\bfseries Support Set.} Ablating the composition of the sampled support mini-batches when training a ResNet-50 on ImageNet for 100 epochs. Our default setup is shaded in green. Increasing the size of the support set improves performance. However, when sampling a fixed number of instances, it is preferable to sample many classes with a few images per class, rather than few classes with many images per class.}
    \label{tb:support-set}
\end{table}

\paragraph{Prediction head.}
As noted in Section~\ref{sec:implementation}, we include a prediction head to facilitate comparison to previous work~\cite{grill2020bootstrap}, where it was suggested as a mechanism to prevent representation collapse. Table~\ref{tb:prediction-head} illustrates that this is not needed when pre-training with \putalg, and in fact the performance of \putalg is marginally better when the prediction head is omitted during pre-training.
\begin{table}[h]
    \centering
    {\small
    \begin{tabular}{lcc}
        & \multicolumn{2}{c}{\bf\small Top 1}\\
        & \bf\small 100 epochs & \bf\small 200 epochs \\\toprule
        \rowcolor{_fbteal3}
        With Prediction Head & 73.9 & 75.0 \\
        Without Prediction Head & 74.2 & 75.2 \\\bottomrule
    \end{tabular}}
    \caption{{\bfseries Prediction Head.} Examining the effect of the prediction-head when training a ResNet-50 on ImageNet and 10\% of the training set is labeled. Our default setup is shaded in green. Unlike self-supervised methods that collapse without a prediction head~\cite{grill2020bootstrap,chen2020exploring}, \putalg still converges without a prediction head, as predicted by the theoretical result Proposition~\ref{prop:collapse}.}
    \label{tb:prediction-head}
\end{table}

\paragraph{ME-Max regularization.}
Recall that \putalg pre-training uses a cross-entropy loss with sharpened targets to encourage representations of different views of the same image to be consistent (reducing cross-entropy), and it also uses the mean-entropy maximization regularizer to maximize the entropy of the average prediction, computed across the unlabeled samples in the mini-batch. Table~\ref{tb:me-max} illustrates the effect of training with only the cross-entropy term and disabling the \textsc{me-max} regularization. While the impact is more pronounced in the setting with only 1\% labeled data, using \textsc{me-max} regularization improves performance in all cases.
\begin{table}[h]
    \centering
    {\small
    \begin{tabular}{lcc}
        & \multicolumn{2}{c}{\bf\small Top 1}\\
        & \bf\small 1\% & \bf\small 10\% \\\toprule
        \rowcolor{_fbteal3}
        With \textsc{me-max} & 63.8 & 73.9 \\
        Without \textsc{me-max} & 52.9 & 73.6 \\\bottomrule
    \end{tabular}}
    \caption{{\bf ME-Max Regularization.} Examining the effect of the \textsc{me-max} regularizer when training a ResNet-50 on ImageNet for 100 epochs. Our default setup is shaded in green. The \textsc{me-max} regularizer is especially helpful in the 1\% label setting, but only provides a marginal improvement in the 10\% label setting.}
    \label{tb:me-max}
\end{table}

\paragraph{Small batch training.}
Our default \putalg implementation runs on 64 GPUs, with a batch-size of 4096 unlabeled images and a supervised support mini-batch of 6720 images, comprising 960 classes and 7 images per class.
We observe that \putalg can also be effectively trained with small batch sizes as well.
Table~\ref{tb:batch-size} ablates the effect of the batch size when training on 8 \textsc{nvidia} V100-16G GPUs, when 10\% of the training set is labeled.
For this small-batch experiment, we set the unsupervised batch size to 256 and attempt to use as large a support set as is possible on 8 GPUs, since the ablation in Table~\ref{tb:support-set} shows that larger supports lead to better performance.
Following a roughly square-root scaling of the learning-rate (relative to the large-batch default setup), we linearly warmup the learning-rate from $0.3$ to $1.2$ during the first 10 epochs of pre-training, and decay it following a cosine schedule thereafter.
We also disable \textsc{me-max} regularization for the small batch experiment, since it is not obvious, a priori, that such regularization will be effective for small batches.
All other settings are kept fixed.
Table~\ref{tb:batch-size} demonstrates that \putalg can still achieve good performance with small batches after only 100 epochs of pre-training on 8 GPUs.

\begin{table}[h]
    \centering
    {\small
    \begin{tabular}{l l c c c}
        & & \multicolumn{2}{c}{\bf\small Support Set} \\
        \bf\small GPUs & \bf\small Batch Size & Classes & Imgs.~per Class & \bf\small Top 1 \\\toprule
        8 V100 & 256 & 560 & 3 & 70.2 \\
        64 V100 & 4096 & 448 & 8 & 70.1 \\
        \rowcolor{_fbteal3}
        64 V100 & 4096 & 960 & 7 & 73.9 \\\bottomrule
    \end{tabular}}
    \caption{{\bfseries Batch Size.} Examining the effect of the batch size when training a ResNet-50 on ImageNet for 100 epochs and 10\% of the training set is labeled. \putalg still achieves good performance after only 100 epochs of pre-training with small batch sizes on 8 \textsc{nvidia} V100-16G GPUs.}
    \label{tb:batch-size}
\end{table}

\section{Discussion}
By leveraging a small labeled support set during pre-training, \putalg achieves competitive classification accuracy for semi-supervised problems and requires significantly less training  than previous works. \putalg also provably avoids collapsing solutions, a common challenge in self-supervised approaches.

\putalg can be interpreted as a neural network architecture with an external memory that is trained using the \emph{assimilation \& accommodation} principle~\cite{piaget1964cognitive}. During assimilation, \putalg updates the representations of new observations so that they are easily described by its external memory (or schemata), while during accommodation, \putalg updates its external memory to account for the new observations.

The use of a supervised support set has some practical advantages as well, since it enables the model to learn efficiently.
However, it remains an interesting question to see if one can learn competitive representations in this framework using only instance supervision and more flexible memory representations. We plan to investigate those directions in future work.

{\small
\bibliographystyle{ieeetr}
\bibliography{refs.bib}
}

\onecolumn
\appendix
{
\section*{\LARGE{Appendix}}
}

\section{Implementation Details}
\label{apndx:implementation}

\paragraph{Sampling the support mini-batches.}
In each iteration, \putalg randomly samples a small support mini-batch from the set of available labeled samples to compute the unsupervised consistency loss.
Specifically, these support samples are used to determine the soft pseudo-labels for the unlabeled image views.
To construct the support mini-batch in each iteration, we first sample a subset of classes, and then sample an equal number of images from each sampled class.
Notably, we sample images \emph{with replacement}.
Therefore, while images in the same support mini-batch in a given iteration are always unique, some of the images may be re-sampled in the subsequent iteration's support mini-batch.
This decision was made to simplify the implementation, although it is possible that epoch-based sampling of the support mini-batches (i.e., iterating through labeled samples with random reshuffling) could lead to improved performance.

\paragraph{Projection \& prediction heads.}
The projection head is a 3-layer MLP with ReLU activations, consisting of three fully-connected layers of dimension $2048$, and Batch Normalization applied to the hidden layers.
The prediction head is a 2-layer MLP with ReLU activations, consisting of two fully-connected layers.
The hidden layer has dimension $512$, and the output layer has dimension $2048$.
Batch Normalization is applied to the input of the prediction head as well as to the hidden layer.
The architectures of these projection and prediction heads are similar to those used in previous works on self-supervised learning~\cite{grill2020bootstrap, chen2020exploring, chen2020big}.

\paragraph{Fine-tuning details.}
Following~\cite{chen2020big}, we fine-tune a linear classifier from the first layer of the projection head in the pre-trained encoder $f_\theta$, and initialize the weights of the linear classifier to zero.
Specifically, we simultaneously fine-tune the encoder/classifier weights by optimizing a supervised cross-entropy loss on the small set of available labeled samples.
We do not employ weight-decay during fine-tuning, and only make use of basic data augmentations (random cropping and random horizontal flipping).
Following the experimental protocol of BYOL~\cite{grill2020bootstrap}, we sweep the learning rate $\{0.01, 0.02, 0.05, 0.1, 0.2\}$ and the number of epochs $\{30, 50\}$.
Similarly to BYOL, to avoid performing parameter selection on the ImageNet validation set (used for reporting), we use a local validation set (12000 images from the ImageNet train set).
Optimization is conducted using SGD with Nesterov momentum.
We use a momentum value of 0.9 and a batch size of 1024.
All results are reported using a single center-crop.

\begin{figure}[b]
    \begin{center}
      \includegraphics[width=0.4\linewidth]{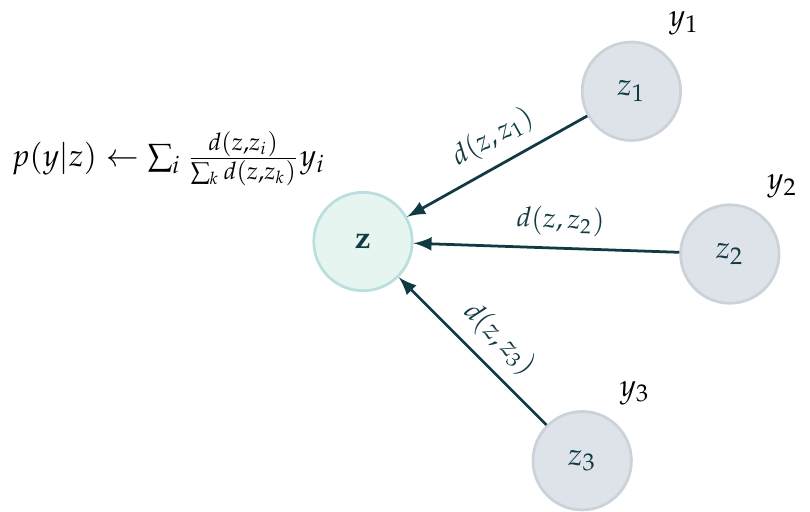}
    \end{center}
    \caption{{\bfseries Soft Nearest Neighbours classifier $\pi_d$.} For a $K$-way classification problem, and a scalar-valued similarity function $d(\cdot, \cdot) \geq 0$, the similarity classifier assigns a soft pseudo-label $y \in [0,1]^K$ to a representation $z$ by measuring its similarity to a set of labeled representations $\{z_i\}_i$ with class labels $\{y_i \in [0,1]^K\}_i$. The soft pseudo-label $y$ is a weighted average of the labels $\{y_i\}_i$, with labels corresponding to more similar representations assigned larger weights.}
    \label{fig:snn}
\end{figure}
\paragraph{Nearest neighbours classifier.}
We also report additional results without fine-tuning the encoder.
Specifically, the \putalg-\textsc{nn} results in Table~\ref{tb:resnet50_results} are reported by directly applying a soft nearest neighbours classifier to the pre-trained representations.
To determine a class prediction for an image $\bf x$, we compare its representation, $z = f_\theta({\bf x})$, to the representations of the available labeled training samples, ${\bf z_\mathcal{S}} \in \R^{M\times d}$, and subsequently choose the class label with the highest probability under the similarity classifier; i.e., ${\text{argmax}}_{k \in [1000]} \left[\pi_d\left(z, {\bf z_\mathcal{S}}\right)\right]_k$.
All results are reported using a single center-crop.
Figure~\ref{fig:snn} provides a schematic of the nearest neighbours classifier in an illustrative example with only only three labeled training images.

\paragraph{Momentum.}
When using momentum in our experiments, unless otherwise specified, we implicitly refer to classical momentum, commonly referred to as heavy-ball or Polyak momentum, given by
\begin{align}
\begin{split}
    v_{t+1} =&\ \beta v_{t} - \eta_t \frac{1}{\lvert \mathcal{B} \rvert} \sum_{x \in \mathcal{B}} \nabla_\theta \ell(x, \theta_t)\\
    \theta_{t+1} =&\ \theta_{t} + v_{t+1},
\end{split} \label{eq:mom}
\end{align}
where $\beta \geq 0$ is the momentum parameter and $\eta_t \geq 0$ is the learning rate.
The model parameters are denoted by $\theta$ and the velocity buffer is denoted by $v$. 
Note that in some deep learning frameworks, such as PyTorch and Tensorflow, the update is instead written
\begin{align}
\begin{split}
    v_{t+1} =&\ \beta v_{t} + \frac{1}{\lvert \mathcal{B} \rvert} \sum_{x \in \mathcal{B}} \nabla_\theta \ell(x, \theta_t)\\
    \theta_{t+1} =&\ \theta_{t} - \eta_t v_{t+1}.
\end{split} \label{eq:pytorch-mom}
\end{align}
Specifically, in eq.~\eqref{eq:pytorch-mom}, the learning rate is not incorporated into the velocity buffer update.
Thus, under a trivial re-parameterization, the eq.~\eqref{eq:pytorch-mom} implementation can be interpreted as classical momentum with a time-varying momentum schedule, given by $\{\beta \frac{\eta_{t}}{\eta_{t-1}} \}_{t > 0}$.
Thus, training with learning rate warmup can result in momentum values $> 1$ during the warmup phase, leading to instability early on in training.
Additionally, note that training using the implementation of momentum SGD in eq.~\eqref{eq:pytorch-mom} with an adaptive learning rate, e.g., as prescribed by the LARS optimizer, can lead to drastically different momentum values at consecutive iterations.
However, it is worth pointing out that the LARS optimizer provided in the popular \textsc{nvidia} \textsc{apex} package wraps around the optimizer, and applies learning-rate adaptation by directly scaling the gradient before the optimization step.
Therefore, using the \textsc{nvidia} \textsc{apex} implementation of LARS with the PyTorch implementation of momentum SGD, without accounting for the subtle implementation differences of PyTorch momentum, produces an odd hybrid of equations~\eqref{eq:mom} and~\eqref{eq:pytorch-mom}.
In our experiments, we use the original version of classical momentum with a constant momentum parameter (i.e., equation~\eqref{eq:mom}), and observe a non-trivial improvement in performance, especially when coupled with LARS adaptation.

\paragraph{Multi-Crop.}
Figure~\ref{fig:method} illustrates the \putalg method when generating two views of each unlabeled image, however, as mentioned in Section~\ref{sec:implementation}, we use the multi-crop data-augmentation of SwAV~\cite{caron2020unsupervised} to generate more than two views of each image in all of our experiments.
Given an unlabeled image, we generate several views of that image by taking two large crops ($224\times224$) and six small crops ($96\times96$).
We use the {\tt RandomResizedCrop} method from the {\tt torchvision.transforms} module in PyTorch.
The two large-crops (global views) are generated with scale $(0.14, 1.0)$, and the six small-crops (local views) are generated with scale $(0.05, 0.14)$, following the original implementation in~\cite{caron2020unsupervised}.

When computing the \putalg loss, each small crop has two positive views (the two global views), and each large crop has one positive view (the other global view).
Specifically, let ${\bf x} \in \R^{n\times(3\times H\times W)}$ denote a mini-batch  of $n$ unlabeled images.
For each image ${\bf x}_i$ in the mini-batch, we generate two large crop views, ${\bf x}^{(1)}_i, {\bf x}^{(2)}_i \in \R^{3\times224\times224}$, and six small crop views, ${\bf x}^{(3)}_i, \ldots, {\bf x}^{(8)}_i \in \R^{3\times96\times96}$.
Let $z^{(1)}_i, \ldots, z^{(8)}_i \in \R^d$ denote the representations computed from ${\bf x}^{(1)}_i, \ldots, {\bf x}^{(8)}_i$ respectively, and let $p^{(1)}_i, \ldots, p^{(8)}_i$ denote the predictions for representations $z^{(1)}_i, \ldots, z^{(8)}_i$ respectively.
Lastly, let $\overline{p} \defeq \frac{1}{8n}\sum^n_{i=1}\sum^8_{k=1}\rho(p^{(k)}_i)$ denote the average of the sharpened predictions (recall $\rho(\cdot)$ is the sharpening function defined in Section~\ref{sec:methodology}).
The overall \putalg objective to be minimized is
\begin{equation}
    \label{eq:multicrop_objective}
    \frac{1}{8n} \sum^n_{i=1} \left(H(\rho(p^{(1)}_i), p^{(2)}_i) + H(\rho(p^{(2)}_i), p^{(1)}_i) + \sum^8_{k=3} H\left(\frac{\rho(p^{(1)}_i) + \rho(p^{(2)}_i)}{2}, p^{(k)}_i\right)\right) - H(\overline{p}).
\end{equation}
In equation~\eqref{eq:multicrop_objective}, $p^{(1)}_i$ and $p^{(2)}_i$ correspond to the two large crop views, and $p^{(3)}_i,\ldots,p^{(8)}_i$ correspond to the six small crop views.
Thus, from equation~\eqref{eq:multicrop_objective}, the target for $p^{(1)}_i$ is the sharpened positive view prediction $\rho(p^{(2)}_i)$, and similarly, the target for $p^{(2)}_i$ is the sharpened positive view prediction $\rho(p^{(1)}_i)$.
For the small views, $p^{(3)}_i,\ldots,p^{(8)}_i$, we use both $\rho(p^{(1)}_i)$ and $\rho(p^{(2)}_i)$ as positive view predictions and average those to produce a single target.
This is similar to the use of multicrop in SwAV~\cite{caron2020unsupervised}.
While the multi-crop augmentation makes the notation in equation~\eqref{eq:multicrop_objective} a little cumbersome, note that this objective is nearly identical to the objective in equation~\eqref{eq:objective}, except that~\eqref{eq:multicrop_objective} also includes a sum over the small crop-views, $\sum^8_{k=3}(\cdots)$.

Intuitively, by only using the large crops as positive samples (note that small crops are never positive samples for any of the other views), the method learns a global-to-local feature mapping, which maps local features in the small crops to global features in the large crops.
The multi-crop augmentation is in fact an essential component of the \putalg algorithm.
As will be shown in Appendix~\ref{apndx:cifar10}, the multi-crop augmentation strategy in \putalg is not only important when training on large internet images, containing possibly obfuscated objects at various scales, such as ImageNet~\cite{russakovsky2015imagenet}, but is also important for small object-centric images, such as CIFAR10~\cite{krizhevsky2009learning}.
This observation suggests that the benefit of ``local-to-global'' matching induced by the multi-crop augmentation strategy in \putalg goes beyond simply inducing obfuscation or scale invariant image representations.

\section{Comparison to Supervised Learning}
\begin{figure}[h]
    \centering
    \begin{subfigure}{0.32\linewidth}
        \centering
        \includegraphics[width=\linewidth]{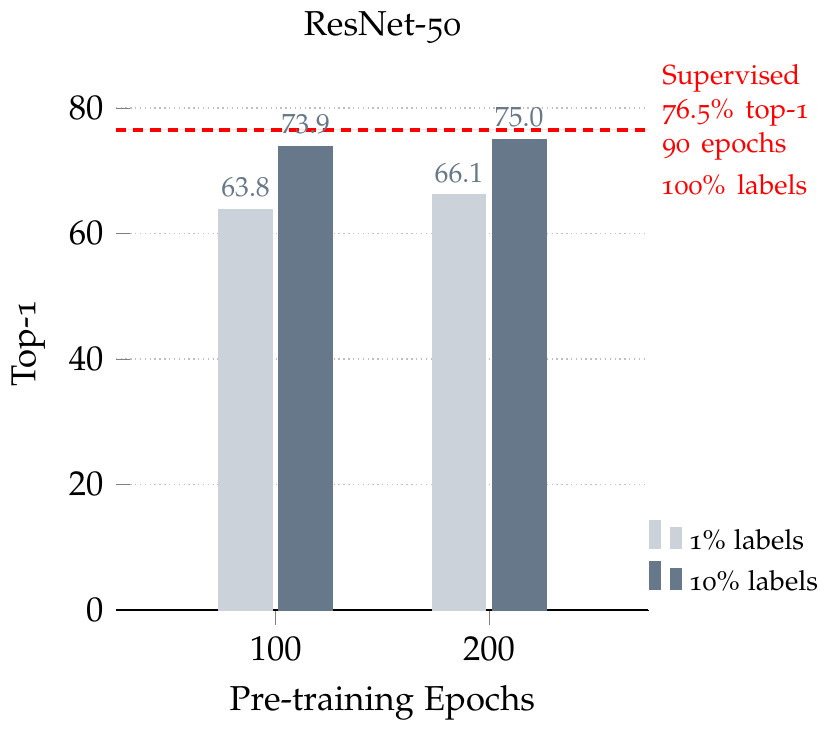}
    \end{subfigure}
    \begin{subfigure}{0.32\linewidth}
        \centering
        \includegraphics[width=\linewidth]{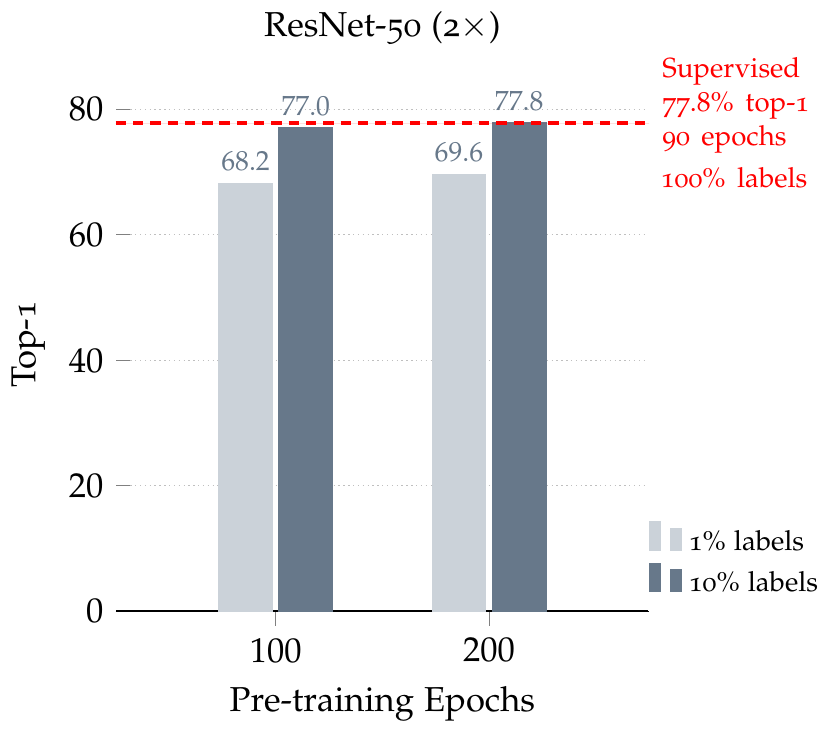}
    \end{subfigure}
    \begin{subfigure}{0.32\linewidth}
        \centering
        \includegraphics[width=\linewidth]{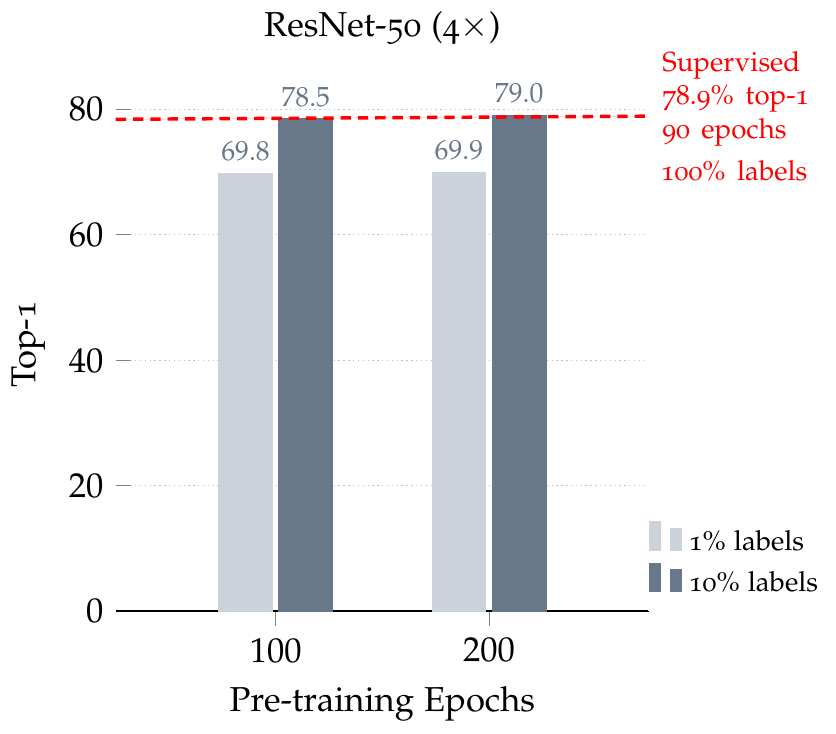}
    \end{subfigure}
    \caption{Comparing ResNet architectures trained using \putalg on ImageNet, with only a small fraction of the training instances labeled, to the same ResNet architectures trained in a fully supervised manner on ImageNet with \emph{all} training instances labeled. Supervised models are reported from SimCLR~\cite[Appendix B.3]{chen2020simple}, and ablated over the same data-augmentations used to train \putalg. We report results for the best supervised model found by~\cite{chen2020simple} over the data-augmentation sweep. When training with ResNet-50 (2$\times$) and ResNet-50 (4$\times$) architectures, \putalg matches the performance of fully supervised learning. Specifically, \putalg is the first method to, with only 10\% of training instances labeled, match fully supervised learning on ImageNet with 100\% of training instances labeled, using the same architecture, and without distilling from a larger teacher model. Notably, this result is achieved with only 200 epochs of training.}
    \label{fig:supervised}
\end{figure}
\begin{wrapfigure}{r}{0.4\textwidth}
    \centering
    \vspace{-1em}
    \includegraphics[width=0.9\linewidth]{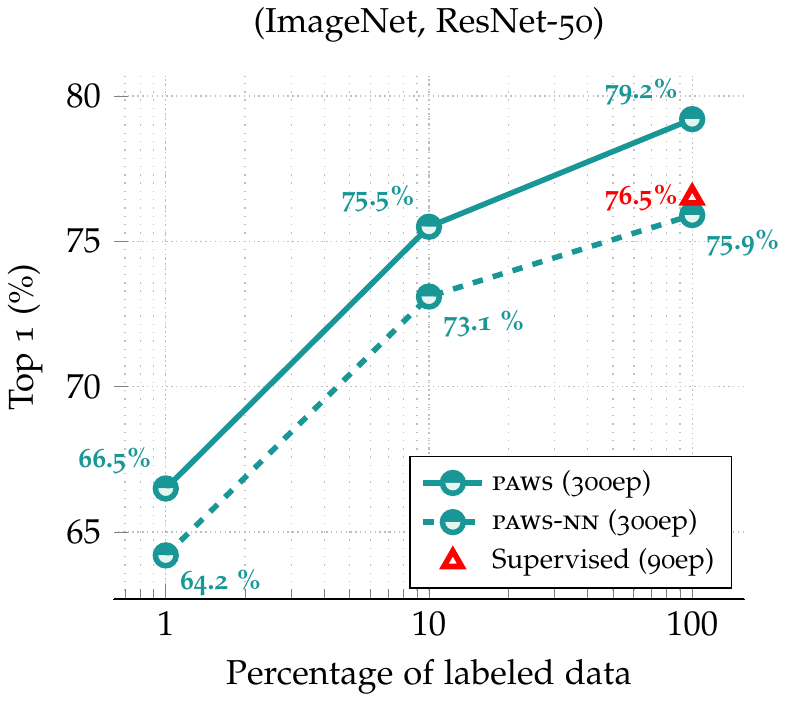}
    \caption{Examining \putalg scaling when training a ResNet-50 on ImageNet with various percentages of labeled data. \putalg-\textsc{nn} refers to performing nearest-neighbour classification directly using the \putalg-pretrained representations, with the labeled training samples as support, while \putalg refers to fine-tuning a classifier using the available labeled data after \putalg-pretraining. Supervised models are reported from SimCLR~\cite[Appendix B.3]{chen2020simple}, and ablated over the same data-augmentations used to train \putalg. We report results for the best supervised model found by~\cite{chen2020simple} over the data-augmentation sweep. When trained with 100\% of the available labels, \putalg surpasses fully supervised learning and produces representations that are well calibrated for non-parametric classification (\putalg-{\sc nn}).}
    \label{fig:supervised-100p}
\end{wrapfigure}

Figure~\ref{fig:supervised} compares \putalg semi-supervised training to supervised learning with the same architecture using a standard cross-entropy loss.
The supervised baseline is trained on the full set of ImageNet labels, whereas the \putalg result is obtained by pre-training (and fine-tuning) with access to only a small fraction of the ImageNet labels.
The supervised models are reported from SimCLR~\cite[Appendix B.3]{chen2020simple}, where they are swept over the number of training epochs $\{90, 500, 1000\}$, and ablated over the data-augmentations used in \putalg pre-training $\{\text{crop/flip}, \text{crop/flip}+\text{color distortion}, \text{crop/flip}+\text{color distortion}+\text{Gaussian blur}\}$.
Figure~\ref{fig:supervised} reports results for the best supervised model found by~\cite{chen2020simple}, which corresponds to 90 epochs of training with random crop/flip for the ResNet-50, and 90 epochs of training with random crop/flip+color distortion for the wider ResNets.
When training with ResNet-50 (2$\times$) and ResNet-50 (4$\times$) architectures, \putalg matches the performance of fully supervised learning.
Specifically, \putalg is the first method to, with only 10\% of training instances labeled, match fully supervised learning on ImageNet with 100\% of training instances labeled, using the same architecture, and without distilling from a larger teacher model.
Notably, this result is achieved with only 200 epochs of training.
However, as a word of caution, this experiment should only be interpreted as a type of ablation, since the performance of supervised learning models can likely be improved by incorporating additional advanced supervised augmentation strategies like Mixup~\cite{zhang2017mixup}, CutMix~\cite{yun2019cutmix}, and AutoAugment~\cite{cubuk2019autoaugment}, which simultaneously learns a data-augmentation policy during training.

\section{Additional Experiments --- CIFAR10}
\label{apndx:cifar10}
We also evaluate the \putalg pre-training scheme on the CIFAR10~\cite{krizhevsky2009learning} dataset using a single {\sc nvidia} V100 GPU.
We first pre-train a network using \putalg on CIFAR10 with access to 4000 labels, and then report the nearest-neighbour classification accuracy on the test set using the 4000 labeled training images as support.
On CIFAR10 we only report \putalg-{\sc nn}, and do not fine-tune a linear classifier on top of the network.
For details on the Nearest Neighbours classifier, see Appendix~\ref{apndx:implementation}.
\begin{figure}[t]
\begin{subfigure}{0.49\linewidth}
    \footnotesize
    \centering
    {\small {WideResNet-28-2, CIFAR10, 4000 labels}\\[2mm]
    \begin{tabular}{l r c}
        \bf\small Method & \bf\small Epochs & \bf\small Top-1 \\\toprule
        Supervised Learning with full dataset~\cite{pham2020meta} & 1000 & 94.9 $\pm$ 0.2 \\\midrule
        \multicolumn{3}{l}{\footnotesize\itshape Methods using label propagation:}\\[1mm]
        Temporal Ensemble~\cite{laine2016temporal} & 300 & 83.6 $\pm$ 0.6 \\
        Mean Teacher~\cite{tarvainen2017mean} & 300 & 84.1 $\pm$ 0.3 \\
        VAT + EntMin~\cite{miyato2018virtual} & 123 & 86.9 $\pm$ 0.4 \\
        LGA + VAT~\cite{jackson2019semi} & -- & 87.9 $\pm$ 0.2 \\
        ICT~\cite{verma2019interpolation} & 600 & 92.7 $\pm$ 0.0 \\
        MixMatch~\cite{berthelot2019mixmatch} & -- & 93.8 $\pm$ 0.1 \\
        ReMixMatch~\cite{berthelot2019remixmatch} & -- & 94.9 $\pm$ 0.0 \\
        EnAET~\cite{wang2019enaet} & 1024 & 94.7 $\pm$ \_\_\_ \\
        UDA~\cite{xie2019unsupervised, pham2020meta} & 2564 & 94.5 $\pm$ 0.2 \\
        FixMatch~\cite{sohn2020fixmatch} & -- & 95.7 $\pm$ 0.1 \\
        MPL~\cite{pham2020meta} & 2564 & \bf 96.1 $\pm$ 0.1 \\\midrule
        \multicolumn{3}{l}{\footnotesize\itshape Non-parametric classification:}\\[1mm]
        \rowcolor{_fbteal3}
        \bf \putalg-\textsc{nn} & \bf 600 & \bf 96.0 $\pm$ 0.2 \\\bottomrule
    \end{tabular}}
    \caption{}
    \label{tb:cifar10_results}
\end{subfigure}
\begin{subfigure}{0.5\linewidth}
    \footnotesize
    \centering
    {\small {Additional Architectures, CIFAR10, 4000 labels}\\[2mm]
    \begin{tabular}{l l c c c}
        \bf\small Method & \bf\small Architecture & \bf\small Params & \bf\small Epochs & \bf\small Top-1 \\\toprule
        SimCLRv2~\cite{chen2020big} & ResNet-200 (+SK) & 95M & 800 & \bf 96.0 \\
        SimCLRv2~\cite{chen2020big} & ResNet-18 (+SK) & 12M & 800 & 92.1 \\\midrule
        \multicolumn{5}{l}{\footnotesize\itshape Non-parametric classification:}\\[1mm]
        \rowcolor{_fbteal3}
        \bf \putalg-\textsc{nn} & WideResNet-28-2 & \bf 1.5M & \bf 600 & \bf 96.0 \\\bottomrule
    \end{tabular}}\vspace{5.5mm}
    {\small PAWS Training Cross-Entropy Loss}\\[2mm]
    \includegraphics[width=0.8\linewidth]{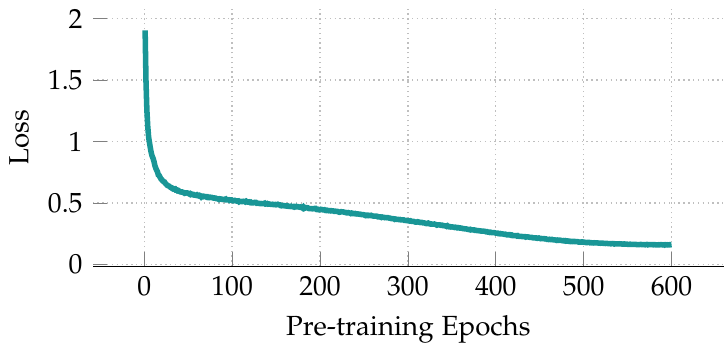}
    \caption{}
    \label{tb:cifar10_results_simclrv2}
\end{subfigure}
\caption{Training a WideResnet-28-2 on CIFAR10.*For label propagation methods, the number of epochs is counted with respect to the unsupervised mini-batches. *For Meta Pseudo-Labels (MPL), the number of epochs only includes the student-network updates, and does not count the additional 1,000,000 teacher-network updates (computationally equivalent to roughly an additional 2564 epochs) that must happen sequentially (not in parallel) with the student updates. \putalg-\textsc{nn} refers to performing nearest-neighbour classification directly using the \putalg-pretrained representations, with the 4000 labeled training samples as support. We report the mean top-1 accuracy and standard deviation across 5 seeds for the 4000 label split.}
\end{figure}

\paragraph{Implementation details.}
We adopt similar hyper-parameters to the ImageNet experiments.
Specifically, for pre-training, we use the LARS optimizer with a momentum value of 0.9, weight decay $10^{-6}$, cosine-similarity temperature of $\tau=0.1$, and target sharpening temperature of $T=0.25$.
To construct the different image views, we use the multi-crop  strategy,  generating  two large crops ($32 \times 32$),  and  six  small  crops  ($18 \times 18$) of each unlabeled image. 
We use the {\tt RandomResizedCrop} method from the {\tt torchvision.transforms} module in PyTorch.
The two large-crops (global views) are generated with scale $(0.75, 1.0)$, and the six small-crops (local views) are generated with scale $(0.3, 0.75)$.
We use a batch-size of 256 and linearly warm-up the learning rate from 0.8 to 3.2 during the first 10 epochs of pre-training, and decay it following a cosine schedule thereafter.
To construct the support mini-batch at each iteration, we also randomly sample 640 images, comprising 10 classes and 64 images per class, from the labeled set, and apply label smoothing with a smoothing factor of $0.1$.
For all sampled images (both unlabeled images and support images) we apply the basic set of SimCLR data-augmentations, specifically, random crop, horizontal flip, and color distortion (but no Gaussian blur).
However, in contrast to the ImageNet setup, we also generate two views of each sampled support image.
On CIFAR10 we find it much easier for the network to learn to classify the images than to perform instance discrimination.
Generating two views of each sampled support image helps the network improve its instance discrimination ability and produce representations that are invariant to the data-augmentations used for training.

The encoder $f_\theta$ in our experiments is a WideResNet-28-2~\cite{zagoruyko2016wide} trunk without dropout, containing a 3-layer MLP projection head, consisting  of three fully-connected layers of dimension 128, and Batch Normalization applied to the hidden layers.
To simplify the implementation, we do not include a prediction head after the projection head.
As shown in Table~\ref{tb:prediction-head} on ImageNet, \putalg pre-training works well without a prediction head, and we find this to be true on CIFAR10 as well.

Following pre-training, we freeze the batch-norm layers, and fine-tune the trunk of the network for 180 optimization steps on the available labeled samples using the supervised contrastive loss of~\cite{assran2020recovering}, and do not apply any data-augmentations during this phase.
The point of these few optimization steps is to tighten the representation clusters of the labelled training samples before using them as support to classify the test images.
For this phase, we use momentum SGD with a batch-size of 640 (comprising 64 images from 10 classes), and sample the mini-batches with replacement; i.e., while images in the same mini-batch in a given iteration are always unique, some of the images may be re-sampled in the subsequent iteration’s mini-batch.
We use a cosine-temperature of $\tau=0.1$, momentum parameter 0.9, a learning rate of $0.1$ with cosine-decay, and no weight-decay.

\paragraph{Results.}
Table~\ref{tb:cifar10_results} compares \putalg-\textsc{nn} to other semi-supervised learning methods trained using identical networks (WideResNet-28-2) on CIFAR10 with access to 4000 labels.
Although the intention here is to simply validate \putalg on another dataset, the observations are similar to ImageNet.
By using the pre-trained representations directly in a nearest neighbour classifier, \putalg can match the state-of-the-art on CIFAR10 with significantly less training.
It is possible that carefully fine-tuning a linear classifier on top of the trunk and incorporating more advanced data-augmentations would further improve performance.
Table~\ref{tb:cifar10_results_simclrv2} compares \putalg-\textsc{nn} to the self-supervised SimCLRv2~\cite{chen2020big} method trained (and fine-tuned) with larger architectures.
The \putalg method achieves superior performance in fewer pre-training epochs, using a residual network containing over $60\times$ fewer parameters.

\section{Alternative Strategies for Non-Collapse}
\label{apndx:theory}

Proposition~\ref{prop:collapse} provides a theoretical guarantee that the proposed method is immune to the trivial collapse of representations.
The underlying principle is that collapsing representations result in high entropy predictions under the non-parametric similarity classifier, but the targets are always low-entropy (because we sharpen them), and so collapsing all representations to a single vector is not a stationary point of the training dynamics.
In this section we demonstrate two simple alternative strategies to guarantee non-collapse of representations without making the target-sharpening assumption.

\subsection{Semi-Supervised Prediction}
If an image in the sampled mini-batch of image views has a class label, then we can directly use that class label as the target for its prediction $p$, rather than using the positive view prediction, $p^+$, as the target.
Under such a scenario, Proposition~\ref{prop:semi-collapse} provides the theoretical guarantee.

\begin{assumption}[Semi-Supervised Image Views]
\label{ass:semi-views}
Each mini-batch of image views contains at least one labeled sample.
\end{assumption}
\begin{proposition}[Non-Collapsing Representations --- Semi-Supervised]
\label{prop:semi-collapse}
Suppose Assumptions~\ref{ass:balanced} and~\ref{ass:semi-views} hold.
If the representations collapse, i.e., $z = z_i$ for all $z_i \in \mathcal{S}$, then $\norm{\nabla H(p^+, p)} > 0$, and the solution is non-stationary.
\end{proposition}
\begin{proof}
The proof is identical to that of Proposition~\ref{prop:collapse}, up to the last step.
At which point, letting $z$ correspond to the labeled instance in the mini-batch of images views, we have that the target $p^+$ is not equal to the uniform distribution because it corresponds to the corresponding ground truth class label.
From which it follows that $p\neq p^+$ and $\norm{\nabla H(p^+, p)} > 0$.
\end{proof}

Note that Proposition~\ref{prop:semi-collapse} is \emph{only} presented as a theoretical alternative strategy to prevent collapse, but is not used in our experiments; instead, we always use the sharpened positive view prediction $p^+$ as the target for the anchor view prediction $p$.

\subsection{Entropy Minimization}
A third possible strategy to guarantee non-collapsing representations without using the target sharpening assumption is to add an entropy minimization term~\cite{grandvalet2006entropy} to the loss.
As shown in the proofs for Propositions~\ref{prop:collapse} and~\ref{prop:semi-collapse}, collapsing representations always result in high-entropy predictions $p$.
These high-entropy predictions result in large non-zero gradients due to the entropy minimization term (which as the name implies is minimized when the entropy is low), and so, just as before, collapsing representations are not stationary points of the training dynamics.
While adding an entropy minimization term to the loss is a conceptually simple strategy, target sharpening is arguably even simpler, and, by Proposition~\ref{prop:collapse}, suffices to guarantee non-collapsing representations, so we do not use entropy minimization in our experiments.

\section{Ethical Considerations}
\label{apndx:ethical}
Increasing model and dataset sizes is a proven approach to improving the performance of image recognition models.
Depending on the intended application, more accurate image recognition models may yield substantial social benefits for society; e.g., improving the quality and safety of systems relying on image recognition.
However, as with any engineering problem, there is no free lunch, and one must not stop grappling with the ethical concerns of more computationally expensive training pipelines, such as potentially larger environmental footprints (depending on the compute cluster used for training) and exclusionary ramifications.
Computationally intensive training pipelines may exclude participation from researchers without access to the computational resources needed to conducted such experiments, which in-turn may lead to slower progress in the field.

The proposed method in this work matches the current state-of-the-art in data-efficient image recognition using considerably smaller models and fewer training epochs.
While our method still benefits from wider and deeper architectures, we demonstrate that the performance of smaller models is not yet saturated, and that research targeting improvements on these smaller models may very well translate to larger-scale settings.

However, generally speaking, we caution against conflating increased computational effort with larger models, since we observe that this relationship is not always linear.
For example, when training a ResNet-50 (2$\times$) for 12 hours (100 epochs) on 64 V100 GPUs, we obtain 68\% top-1 accuracy in the 1\% label setting and 77\%  in the 10\% label setting.
Conversely, when training a smaller ResNet-50 for 17 hours (200 epochs) on 64 V100 GPUs, we obtain 66\% top-1 accuracy in the 1\% label setting and 75\% in the 10\% label setting.

\section{Historical Perspective}
\label{apndx:historical}
Constructivist learning theory---developed a near half-century ago by Jean Piaget and built on notions of schemata put forth by Immannuel Kant---has (surprisingly) withstood the test of time.
Constructivism not only revolutionized school curricula in the 20$^{\text{th}}$ century, but remains to this day a crucial element of many teaching philosophies---placing greater emphasis on spontaneous learning through self-regulation and concrete activities, often under the pseudonym of Project-Based Learning in primary and secondary schools, and Lab-Based Instruction in post-secondary institutions.
At the heart of Constructivism is the idea that every individual possesses mental schemata---representations relating to distinct semantic concepts---and that learning occurs through the process of \emph{assimilation and accommodation}.\footnote{The term schema may be familiar to researchers working with relational database systems, where it has become standard jargon referring to the logical structure of a database (in close relation to its original meaning in psychology).}
During assimilation, the mind adapts its representation of new experiences to fit its existing schemata, while during accommodation, the existing schemata are updated to make sense of new experiences.
In short, Constructivism purports that knowledge is ``constructed'' through self-guided exploration, and that mental representations of semantic concepts in sensorimotor observations are learned by conforming new observations to past experiences and vice versa.

It is of particular interest to us to note that one of Piaget's tenets was that sensorimotor development came about the process of optimizing a non-purposive mental objective using assimilation and accommodation.
Non-purposive learning generally refers to the process of learning without working towards any particular purpose or goal.
As such, non-purposive learning is generally concerned with deriving mental models, or schemata, of sensorimotor observations, under which all new observations can be readily explained in terms of past observations.
Clearly, non-purposive learning is closely related to the idea embodied nowadays by task-agnostic self-supervised pre-training, but differs slightly.
Whereas current task-agnostic self-supervised learning approaches predict inputs from inputs in a fully unsupervised manner, non-purposive learning approaches do not preclude the use of semantic information.
To the contrary, semantic information can be used to aid in the construction of sensorimotor schemata; i.e., non-purposive learning can be unsupervised, semi-supervised, weakly-supervised, or fully supervised.
This paper proposes a non-purposive method for semi-supervised learning.

\paragraph{Criticisms of Constructivist Learning Theory.}
Despite the widespread success of Constructivisim, one of the weaknesses of Piagetian theory is its lack of specificity in describing the mechanisms by which assimilation and accommodation occur to produce mental representations of semantic concepts in sensorimotor observations~\cite{boden1978artificial}.
It is perhaps for this reason that Piaget was especially interested in the emerging field of cybernetics (a precursor to artificial intelligence developed in the 40's by Norbert Wiener) and has gone so far as to say that ``Life is essentially auto-regulation,'' and ``cybernetic models are, so far, the only ones throwing any light on the nature of auto-regulatory mechanisms''~\cite{piaget1971biology}.
Piaget advocated for cybernetic models with great aplomb, ``I wish to urge that we make an attempt to use it''~\cite{bruner1961individual}, and may have attempted to use them himself had it not been for his advanced age.
Unfortunately, despite the clear links to cybernetics, the connection to Constructivism did not readily carry over to artificial intelligence (AI) in the 70's due to the largely symbolic nature of AI approaches at the time; e.g., it was not obvious how to represent the near infinite variations of a hand-drawn curve in a single concise representation (i.e., a schema); an issue which is now largely resolved by gradient-based learning and modern neural network architectures.

\section{Change Log}
\begin{tabular}{r l l}
     April 28, 2021 & {\bf [v1]} & Initial preprint.\\[1mm]
     \pbox{20cm}{May 26, 2021\\} & \pbox{20cm}{{\bf [v2]}\\} & \pbox{20cm}{Corrected the 1\% NN numbers (table 1). Added \putalg training with 100\% labels (fig.7). Corrected\\typo in the description of momentum.}\\[3.5mm]
     \pbox{20cm}{July 29, 2021\\} & \pbox{20cm}{{\bf [v3]}\\} & \pbox{20cm}{Added small-batch results (table 7). Added experiment with 1\% labels and 1000 classes (table 4).\\Expanded related work discussion.}
\end{tabular}
\end{document}